\documentclass{article}

\usepackage[final]{corl_2026}
\makeatletter
\renewcommand{\@conferencelocation}{Austin, Texas, USA}
\makeatother
\usepackage[T1]{fontenc}

\usepackage{amsmath}
\usepackage{amssymb}
\usepackage{amsthm}
\usepackage{booktabs}
\usepackage{graphicx}
\usepackage{subcaption}
\usepackage{xcolor}
\usepackage{multirow}
\usepackage{wrapfig}

\theoremstyle{plain}
\newtheorem{proposition}{Proposition}

\newtheorem*{propositionrestate}{Proposition 1, restated}

\title{Phase-Conditioned Affordance Fields for $SE(3)$-Equivariant Visuomotor Policies}

\newcommand{\paperauthors}{Xinzhe Chen, Liqi Huang, Sihua Ren, Haowen Sun, Mingyang Li, Zeyang Liu, Xingyu Chen, Xuguang Lan}
\author{
  \textbf{Xinzhe Chen, Liqi Huang, Sihua Ren, Haowen Sun}\\
  \textbf{Mingyang Li, Zeyang Liu, Xingyu Chen\thanks{Corresponding authors: \href{mailto:xglan@mail.xjtu.edu.cn}{xglan@mail.xjtu.edu.cn} and \href{mailto:chenxingyu1990@xjtu.edu.cn}{chenxingyu1990@xjtu.edu.cn}}$^{\phantom{*}}$, Xuguang Lan\footnotemark[1]}\\
  National Key Laboratory of Human-Machine Hybrid Augmented Intelligence\\
  Institute of Artificial Intelligence and Robotics\\
  Xi'an Jiaotong University, Xi'an, China\\
  Project page: \url{https://paf-policy.github.io/}
}

\hypersetup{
  pdftitle={Phase-Conditioned Affordance Fields for SE(3)-Equivariant Visuomotor Policies},
  pdfauthor={\paperauthors},
  pdfkeywords={Robot Manipulation, SE(3) Equivariance, Visuomotor Policy}
}

\begin{document}
\maketitle
\setcounter{footnote}{0}
\vspace{-0.15in}

\begin{abstract}
$SE(3)$-equivariance provides a useful geometric inductive bias
for visuomotor policies,
encoding that actions covary with rigid transformations of the
observation.
However,
existing $SE(3)$-equivariant visuomotor policies build this
bias on a global equivariance assumption that treats the entire
observation as governed by a single rigid transformation,
which breaks down in real-world manipulation tasks where the
task-relevant region changes across different objects.
We address this by learning a phase-conditioned affordance
field (PAF) over the observed point cloud,
a per-point score that the $SE(3)$-equivariant policy uses in two complementary
ways.
Its score-weighted centroid serves as a learned,
phase-conditioned anchor that canonicalizes the observation and
the action around the active region,
replacing the fixed gripper-frame reference,
and the score itself is appended as a per-point feature.
This introduces a local geometric inductive bias,
while preserving the policy's global $SE(3)$ equivariance.
We validate on multiple MimicGen tasks under varying spatial
randomization,
and PAF consistently improves the success rate of the
underlying policy.
On multi-object real-robot tasks,
PAF outperforms baselines and keeps its lead under
lighting variations.

\end{abstract}


\section{Introduction}
\label{sec_introduction}

Real-world robotic manipulation requires policies that
generalize across varying object poses and scene layouts.
Visuomotor policies~\citep{florence2022implicit, zhao2023learning}
such as Diffusion Policy~\citep{chi2025diffusion} capture diverse
action distributions from expert demonstrations,
but often degrade under such variations
when demonstrations are scarce.
Equivariant visuomotor policies introduce a geometric bias
for this setting~\citep{wang2025equivariant, tie2025etseed, huang2022equivarianttransporternetwork, Ryu_2024_CVPR, yang2025equibot}.
By encoding observations and decoding actions in a spherical
Fourier space,
Spherical Diffusion Policy (SDP)~\citep{zhu2025se} makes the
predicted action covary with rigid transformations of the
scene,
improving robustness to object pose variation.

However,
SDP relies on a global equivariance assumption that treats the
entire observation as governed by a single rigid
transformation.
This assumption breaks down in multi-object tasks where
objects move independently.
For example,
Figure~\ref{fig:overview} shows that hanging a cup on a rack
requires the policy to focus first on the cup for grasping,
then on the rack for transport,
and finally on the cup handle and rack pole for alignment.
A single scene-level transform cannot select which geometry
governs each phase.
Object-level $SE(3)$-equivariant formulations are also
difficult to apply broadly,
since per-object poses are often unavailable or unreliable
when estimated from raw observations,
and many phases involve inter-object interactions rather than
independent motion.

Our key insight is to address this problem along two
complementary axes,
temporal and spatial.
Along the temporal axis,
the robot's recent motion signals the current phase,
allowing us to decompose a multi-object task into a sequence of
phases.
Along the spatial axis,
each phase has its own active region whose geometry should
govern the action.
We represent this region as an affordance field~\citep{wu2021vat, ju2024robo, nasiriany2025rt, wang2022adaafford},
a per-point score over the observed point cloud that highlights
the task-relevant geometry.
Conditioning the equivariant policy on this field provides
a local geometric inductive bias while retaining the policy's
global $SE(3)$ equivariance.

\begin{figure}[t]
    \centering
    \includegraphics[width=\linewidth]{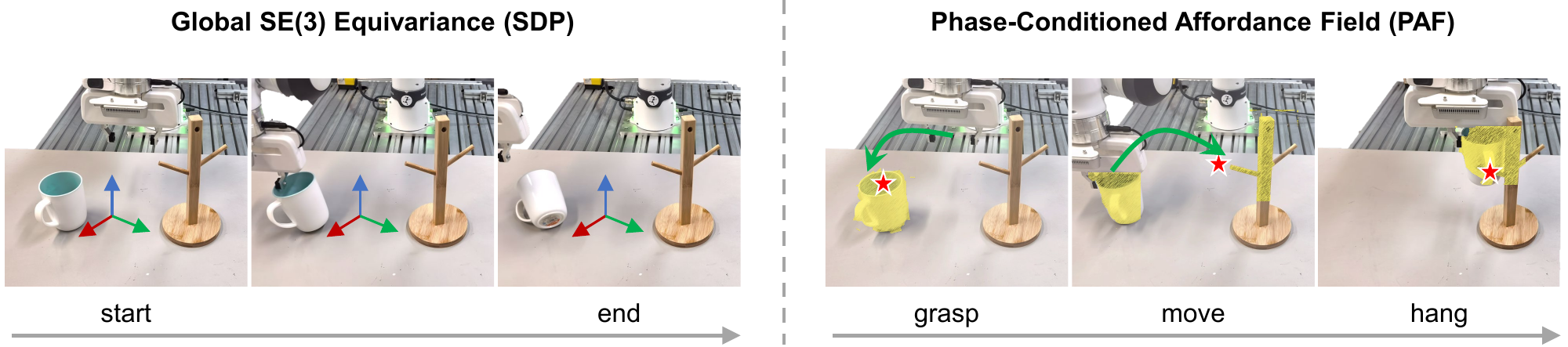}
    \caption{SDP versus PAF on the hang-cup task.
    Left, SDP uses a single global $SE(3)$ frame for the entire
    task.
    Right, PAF predicts affordance scores (yellow) and their
    centroid anchor (red star) across the grasp, move, and hang
    phases.}
    \label{fig:overview}
\end{figure}

We instantiate this insight as Phase-Conditioned Affordance
Fields (PAF).
PAF encodes the recent end-effector and gripper history into
a phase context,
then conditions an invariant message-passing
network~\citep{gilmer2017neural, satorras2021n} on this context
to score every observed point.
The score-weighted centroid acts as a soft anchor that
recenters the observation on the active region,
as illustrated on the right of Figure~\ref{fig:overview}.
The score itself is also appended to the point cloud as a
per-point feature.
We then feed this recentered, score-augmented observation to the
SDP backbone,
whose $SO(3)$-equivariant diffusion network produces actions in
the active-region frame.
Moreover, PAF requires no segmentation labels and is trained with
self-supervised pseudo-labels derived from expert demonstrations.

We evaluate PAF on MimicGen simulation
tasks~\citep{mandlekar2023mimicgen} and on six multi-object real-robot tasks.
Across twelve MimicGen settings,
PAF raises the average success rate from the published SDP
baseline's $48.8\%$ to $56.3\%$,
with the largest gain at the d3 randomization level,
where the average improves from $48.5\%$ to $60.3\%$.
On the real robot,
PAF averages $67.0\%$ success under normal lighting,
$9.0$ points above SDP,
and retains a lead of $10.6$ points under out-of-distribution
lighting.

\section{Related Work}
\label{sec:related_work}

\textbf{Diffusion-based Visuomotor Policies.}
Visuomotor policies are commonly learned from expert
demonstrations~\citep{florence2022implicit, zhao2023learning},
mapping observations to robot actions.
Diffusion Policy~\citep{chi2025diffusion} adopts a diffusion
model that captures multimodal action distributions.
Subsequent variants extend this paradigm with richer scene
representations~\citep{reuss2023goal, ke20253d, ze2025generalizable, liu2025rdt, gervet2023act3d}.
DP3~\citep{ze2024dp3} uses point-cloud observations with a
3D encoder,
exposing scene geometry directly to the diffusion policy and
improving spatial reasoning.
Other variants further inject affordance priors over
task-relevant geometry.
AnchorDP3~\citep{zhao2025anchordp33daffordanceguided} and
AffordDP~\citep{wu2025afforddp}, for instance, draw these
priors from external segmentation or foundation models.
These methods leave the observation-to-action transformation
largely implicit,
so geometric generalization depends on the pose coverage of the
demonstrations.

\textbf{Equivariant Visuomotor Policies.}
Equivariance~\citep{9812146, fuchs2020se, thomas2018tensorfieldnetworksrotation}
provides a geometric inductive bias for visuomotor policies by
encoding the rigid-transformation relationship between
observation and action.
Early work improves sample efficiency for grasping and
pick-and-place~\citep{huang2022equivarianttransporternetwork, zhu2022sampleefficientgrasplearning, 10160728}.
\mbox{Recent} methods combine equivariant networks with diffusion or
trajectory-level policy
learning~\citep{tie2025etseed, yang2025equibot}.
EquiDiff~\citep{wang2025equivariant} applies discrete
$SO(2)$-equivariance to a voxel diffusion policy end-to-end.
Spherical Diffusion Policy~\citep{zhu2025se} attains continuous
$SO(3)$-equivariance through spherical Fourier
features~\citep{s.2018spherical} and translation invariance
through gripper-frame canonicalization of observations and
actions,
together yielding an $SE(3)$-equivariant closed-loop policy.
These methods improve manipulation generalization but tie
action prediction to a fixed or global reference frame,
whereas the task-relevant geometry shifts across manipulation
phases.

\textbf{Local Frames for Visuomotor Policies.}
Other work attaches the policy to a local frame.
EDF~\citep{ryu2023edf} and Diffusion-EDFs~\citep{Ryu_2024_CVPR}
sample a target end-effector pose from an $SE(3)$-equivariant
energy or score model and hand it to a planner,
so they predict goal poses rather than closed-loop actions.
HEP~\citep{zhao2025hep} predicts a keypose translation with a
high-level agent and runs a low-level trajectory policy in the
resulting frame,
and Oriented Affordance Frames~\citep{rana2025oaf} obtain the
frame from external detection, segmentation, and pose estimation
models and compose independently trained sub-policies in it.
In both cases the frame is one predicted or detected pose,
and the policy is split into levels or sub-policies.
PAF builds on the diffusion backbone of SDP but derives the frame
origin from a phase-conditioned affordance field,
whose score-weighted centroid is the anchor and whose values
condition one end-to-end closed-loop policy,
so the frame origin tracks each phase's active region without hierarchy
or external perception while the policy retains global $SE(3)$
equivariance.

\section{Global $SE(3)$ Equivariance}
\label{sec:background_and_problem_formulation}

At time $t$,
the policy observes a state $S_t=(O_t,H_t)$ consisting of a
point cloud $O_t$ and a robot history $H_t$.
The point cloud $O_t=\{(x_i,f_i)\}_{i=1}^{N}$ contains $N$
points,
each with coordinate $x_i\in\mathbb{R}^{3}$ and non-geometric
features $f_i$ such as color.
The robot history
$H_t=\{(e_{\tau},r_{\tau},q_{\tau})\}_{\tau=t-L_h+1}^{t}$
collects end-effector position $e_{\tau}$,
rotation $r_{\tau}$,
and gripper state $q_{\tau}$ over the last $L_h$ steps.
The policy outputs an action trajectory
$A_t=\{a_{\tau}\}_{\tau=t}^{t+T_a-1}$,
where each action
$a_{\tau}=(e_{\tau}^{a},r_{\tau}^{a},q_{\tau}^{a})$ specifies a
translation,
a rotation,
and a gripper command.

A visuomotor policy $\pi_{\theta}$ is $SE(3)$-equivariant if,
for any rigid transformation $g=(R,u)\in SE(3)$,
applying $g$ to the observed state transforms the predicted
action trajectory by the same $g$,
\begin{equation}
    \pi_{\theta}(\cdot\mid g\cdot S_t)
    =
    g\cdot \pi_{\theta}(\cdot\mid S_t),
    \quad\forall\,g\in SE(3).
    \label{eq:policy_equivariance}
\end{equation}

SDP realizes this property in two steps.
Positions are first canonicalized around the current end-effector,
$\bar{x}_i=x_i-e_t$ and $\bar{e}^{a}_{\tau}=e^{a}_{\tau}-e_t$,
which removes the translation part of $g$,
and an $SO(3)$-equivariant backbone then handles the remaining
rotation,
since a global $g=(R,u)$ maps $\bar{x}_i$ to $R\bar{x}_i$.
This construction assumes that a single global $g$ acts on the
entire observed state.
However,
the global assumption fails when a scene contains multiple
objects that move independently,
since no single rigid transformation captures their combined
motion.
A canonicalization reference such as the gripper can absorb at
most one object's translation,
since a point on object $j$ translated by $u_j$ becomes
$\bar{x}_i+u_j-u_e$ when the reference moves by $u_e$,
leaving residual offsets for the remaining objects,
and a global $SO(3)$-equivariant architecture can only enforce
equivariance with respect to one shared scene rotation.
A formal derivation is given in
Appendix~\ref{app:multi_object_derivations}.

\section{Method}
\label{sec:method}

\subsection{Design Principle of PAF}
\label{sec:method_overview}

Localizing equivariance to the task-relevant object would, in
the strictest form, require per-object $SE(3)$ equivariance,
which is difficult to realize because objects interact and their
relative poses couple in the observation.
We instead model manipulation as a sequence of phases,
each focusing on an active region that may be a single object
or a multi-object interaction area.
Intuitively,
the predicted action should be governed by the geometry of the
active region,
while changes to phase-irrelevant geometry should have limited
effect.

Let $\zeta_t$ denote the phase context at time $t$ inferred
from the robot history $H_t$,
and let $\mathcal{R}_t\subseteq\{1,\ldots,N\}$ index the
observed points belonging to the active region in phase
$\zeta_t$.
For a rigid transformation $g_{\mathcal{R}_t}\in SE(3)$ that
acts only on the points indexed by $\mathcal{R}_t$,
we would like
\begin{equation}
    \pi_{\theta}(\cdot\mid g_{\mathcal{R}_t}\cdot S_t,\zeta_t)
    =
    g_{\mathcal{R}_t}\cdot \pi_{\theta}(\cdot\mid S_t,\zeta_t),
    \quad\forall\,g_{\mathcal{R}_t}\in SE(3).
    \label{eq:object_level_equivariance_desideratum}
\end{equation}
Eq.~\eqref{eq:object_level_equivariance_desideratum} localizes
Eq.~\eqref{eq:policy_equivariance} by replacing the global
$g\cdot S_t$ with the region-restricted
$g_{\mathcal{R}_t}\cdot S_t$.
A region-restricted $g_{\mathcal{R}_t}$ alters the geometric
relationships between the active region and the rest of the
scene,
so PAF's prediction does not transform cleanly under
$g_{\mathcal{R}_t}$.
We therefore treat
Eq.~\eqref{eq:object_level_equivariance_desideratum}
as a target that PAF approaches empirically through a learned,
phase-conditioned geometric signal,
rather than as a property it provably satisfies.

\begin{figure}[t]
    \centering
    \includegraphics[width=\linewidth]{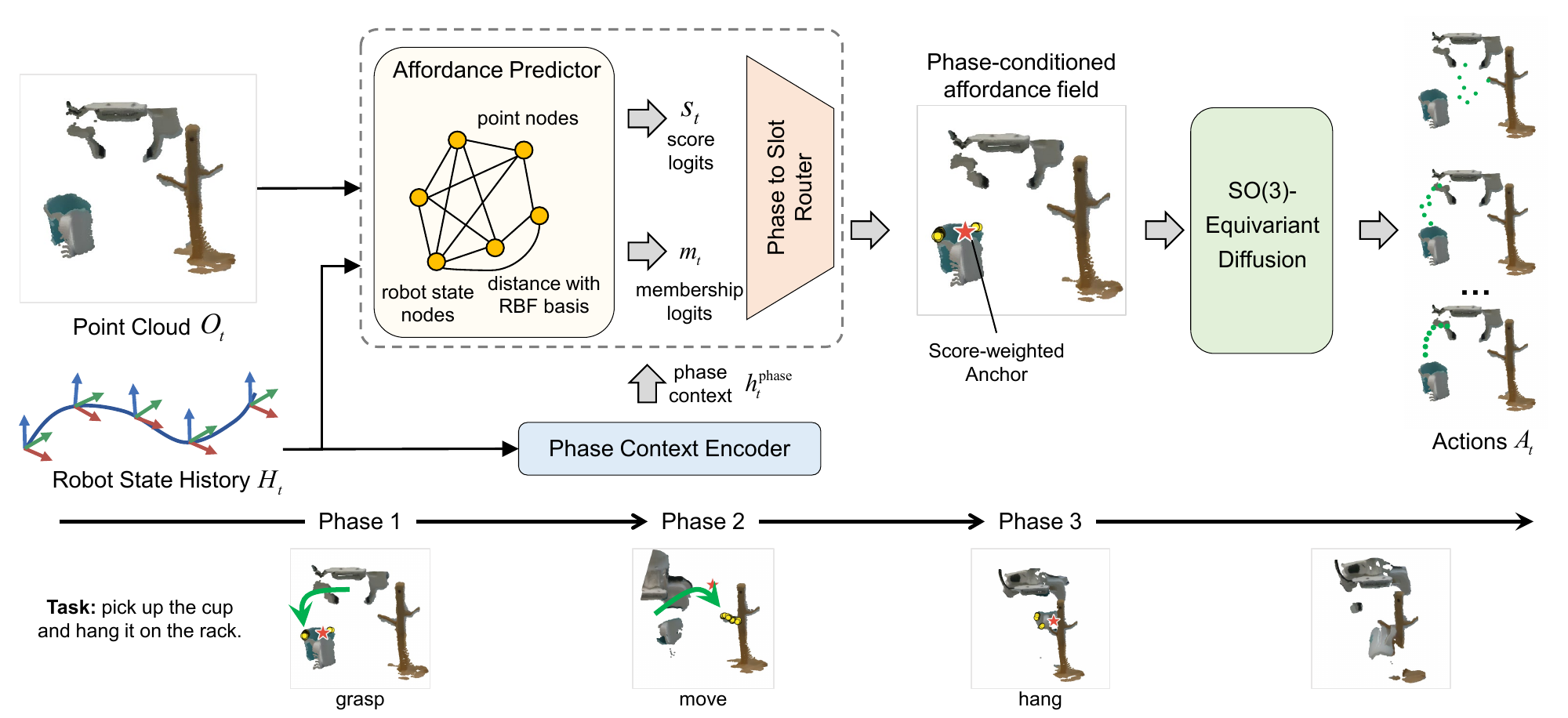}
    \caption{Overview of the PAF framework.}
    \label{fig:method_overview}
\end{figure}

PAF approaches this goal along two complementary axes,
temporal and spatial.
The phase context $h_t^{\mathrm{phase}}$ instantiates $\zeta_t$,
summarizing the current manipulation phase from the robot
history.
The affordance field $\hat{s}_t$ identifies $\mathcal{R}_t$,
highlighting the active-region geometry and defining a local
frame origin at the score-weighted anchor $c_t$.
Figure~\ref{fig:method_overview} illustrates the resulting
framework.
These two signals combine to give the PAF policy
\begin{equation}
    \pi_{\theta}^{\mathrm{PAF}}(\cdot\mid S_t, h_t^{\mathrm{phase}})
    =
    T_{c_t}\cdot \pi_{\phi}(\cdot\mid \tilde{O}_t,\tilde{H}_t),
    \label{eq:paf_policy_structure}
\end{equation}
where $\pi_{\phi}$ is the $SO(3)$-equivariant SDP backbone,
$\tilde{O}_t$ and $\tilde{H}_t$ are the point cloud and history
canonicalized by $c_t$,
with $\hat{s}_t$ appended to $\tilde{O}_t$ only,
and $T_{c_t}$ shifts action positions back by $c_t$.

\subsection{Phase Context Encoding}
\label{sec:method_affordance_field}

\textbf{Phase context vector.}
We form a phase context vector $\xi_t\in\mathbb{R}^{d_{\xi}}$
from six rigid-motion invariant statistics over the recent
robot history $H_t$.
From consecutive end-effector poses we compute the speed
$v_\tau=\|e_\tau-e_{\tau-1}\|$, the gripper change
$\Delta q_\tau=q_\tau-q_{\tau-1}$,
and the angular displacement $\omega_\tau$ between consecutive
orientations.
From triples of poses we compute the speed change
$\Delta v_\tau=v_\tau-v_{\tau-1}$
and the turning angle $\gamma_\tau$ between consecutive
displacement vectors.
Together with the gripper state $q_\tau$,
these statistics form $\xi_t$.
Full formulas are given in
Appendix~\ref{app:phase_feature_computation}.

\textbf{Phase context encoder.}
A learnable phase context encoder
$\Phi_{\mathrm{phase}}:\mathbb{R}^{d_{\xi}}\rightarrow
\mathbb{R}^{d_{\mathrm{phase}}}$
maps $\xi_t$ to the phase context
$h_t^{\mathrm{phase}}=\Phi_{\mathrm{phase}}(\xi_t)$,
which encodes phase-dependent motion patterns such as gripper
interactions and transitions between manipulation phases.
Because $\xi_t$ is invariant,
$h_t^{\mathrm{phase}}$ is also $SE(3)$ invariant,
and it conditions the affordance field predictor and the
phase-to-slot router.

\subsection{Phase-Conditioned Affordance Field}
\label{sec:method_anchor_prediction}

\textbf{Affordance field predictor.}
To determine which points belong to the active region,
each point must encode its local geometry and its relationship
to the end-effector under the current phase context.
We build a $k$-nearest-neighbor (KNN) graph over point
and recent end-effector nodes using Euclidean distance
$d_{ij}=\|p_i-p_j\|$,
where $p_i$ is the coordinate of node $i$,
and encode each distance with a radial basis function (RBF)
basis $\rho_{ij}$.
The phase context is injected into each node feature
by a broadcast linear projection,
and node features are updated over $L_g$ layers of invariant
message passing~\citep{gilmer2017neural, satorras2021n},
\begin{equation}
    h_i^{\ell+1}
    =
    h_i^{\ell}
    +
    \psi_h\!\left(h_i^{\ell},\,
        \frac{1}{|\mathcal{N}(i)|}
        \sum_{j\in\mathcal{N}(i)}
        \psi_e(h_i^{\ell},h_j^{\ell},\rho_{ij})\right),
    \label{eq:invariant_message_passing}
\end{equation}
where $h_i^{\ell}$ is the feature of node $i$ at layer $\ell$,
$\mathcal{N}(i)$ is its KNN neighbor set,
and $\psi_e$, $\psi_h$ are MLP edge and node update functions.
After the final layer,
two MLP heads decode each point-node feature $h_i^{L_g}$ into a
base score logit $s_i\in\mathbb{R}$ and slot membership logits
$m_i\in\mathbb{R}^{K}$.
Graph construction details are given in
Appendix~\ref{app:graph_construction}.

\textbf{Phase-to-slot router.}
Because the active region changes with the manipulation phase,
we introduce $K$ learned slots to model distinct regions and
combine the per-point logits $s_i$, $m_i$ with the phase
context $h_t^{\mathrm{phase}}$ into a phase-conditioned
affordance field.
The slot membership logits $m_i$ assign each point to slots
via $\beta_{i,k}=\mathrm{softmax}_k(m_{i,k})$,
and a phase-to-slot router $R_{\theta}$ selects the active slot
from the phase context,
$\alpha_t=\mathrm{softmax}(R_{\theta}(h_t^{\mathrm{phase}}))$.
The affordance score $\hat{s}_t(i)$ combines the base score
$b_i=\sigma(s_i)$, where $\sigma$ is the sigmoid function,
with the slot weighting,
$\hat{s}_t(i)=b_i\sum_{k=1}^{K}\alpha_{t,k}\beta_{i,k}$.
The anchor $c_t$ is the score-weighted centroid of the point
cloud,
\begin{equation}
    w_i
    =
    \mathrm{softmax}_i\!\left(
        \log(\hat{s}_t(i)+\epsilon)\right),
    \qquad
    c_t=\sum_{i=1}^{N}w_i x_i,
\end{equation}
where $x_i$ is the coordinate of point $i$,
$\epsilon$ is a small constant for numerical stability,
and $w_i$ is the normalized anchor weight.

\textbf{Affordance-enhanced action generation.}
The score $\hat{s}_t$ is appended as a per-point feature,
and the anchor $c_t$ recenters all positional inputs,
$\tilde{x}_i=x_i-c_t$,
$\tilde{e}_{\tau}=e_{\tau}-c_t$,
$\tilde{e}_{\tau}^{a}=e_{\tau}^{a}-c_t$,
while orientations,
gripper states,
and point features are unchanged.
The augmented cloud
$\tilde{O}_t=\{(\tilde{x}_i,f_i,2\hat{s}_t(i)-1)\}_{i=1}^{N}$
and canonicalized history
$\tilde{H}_t=\{(\tilde{e}_{\tau},r_{\tau},q_{\tau})\}_{\tau=t-T_o+1}^{t}$
over the last $T_o$ observation steps
are passed to the SDP backbone~\citep{zhu2025se},
whose $SO(3)$-equivariant denoising network samples
$\tilde{A}_t\sim\pi_{\phi}(\cdot\mid\tilde{O}_t,\tilde{H}_t)$.
At inference,
predicted action positions are shifted back by $c_t$
while orientation and gripper commands are kept as predicted.

\subsection{Training Objective}
\label{sec:method_training}

\textbf{Affordance score.}
We supervise the base score logit $s_i$ with pseudo-labels derived from
the future action position $y_t$ in the current phase,
rather than object annotations.
The pseudo-label assigns high values to points near the target
action position and is damped near the current end-effector.
The score loss $\mathcal{L}_{\mathrm{score}}$ is binary cross
entropy between $s_i$ and the pseudo-label.
The full pseudo-label formulation is given in
Appendix~\ref{app:slot_warmup},
with its scale parameters listed in Table~\ref{tab:appendix_hyperparameters}.

\textbf{Phase and slot supervision.}
A separate phase classifier head maps the phase context
$h_t^{\mathrm{phase}}$ to sub-phase and cycle predictions,
supervised by cross-entropy loss $\mathcal{L}_{\mathrm{phase}}$
against the sub-phase label $\varphi_t$ and cycle index
$\kappa_t$, both derived automatically from the gripper commands
and observed gripper states.
The phase-to-slot router output $\alpha_t$ is supervised by a
slot cross-entropy loss $\mathcal{L}_{\mathrm{slot}}$ against a
slot label that combines $\varphi_t$ with $\kappa_t$.
The per-point slot membership logits $m_i$ receive gradient only
through the DDPM~\citep{NEURIPS2020_4c5bcfec} denoising loss
$\mathcal{L}_{\mathrm{ddpm}}$.
Label definitions are given in Appendix~\ref{app:slot_warmup}.

\textbf{Overall objective.}
The full objective combines $\mathcal{L}_{\mathrm{ddpm}}$
with $\mathcal{L}_{\mathrm{score}}$,
$\mathcal{L}_{\mathrm{phase}}$,
and $\mathcal{L}_{\mathrm{slot}}$,
\begin{equation}
    \mathcal{L}
    =
    \mathcal{L}_{\mathrm{ddpm}}
    +
    \lambda_s\mathcal{L}_{\mathrm{score}}
    +
    \lambda_p\mathcal{L}_{\mathrm{phase}}
    +
    \lambda_z\mathcal{L}_{\mathrm{slot}},
    \label{eq:overall_training_objective}
\end{equation}
where $\lambda_s,\lambda_p,\lambda_z$ are scalar loss weights.

\subsection{Equivariance Analysis of PAF}
\label{sec:method_theoretical_analysis}

Consider a global rigid transformation $g=(R,u)\in SE(3)$
applied to all observed point and end-effector poses.
The affordance scores are invariant under $g$,
so the score-weighted anchor $c_t$ transforms with the scene.
The following proposition states this formally.

\begin{proposition}
Let $g=(R,u)\in SE(3)$ act on all observed point and end-effector
poses,
and let $\hat{s}_t^{\,\prime}(i)$ and $c_t^{\prime}$ denote the
affordance score and anchor computed from the transformed
observation $g\cdot S_t$.
Then $\hat{s}_t^{\,\prime}(i)=\hat{s}_t(i)$ (score invariance)
and $c_t^{\prime}=Rc_t+u$ (anchor equivariance).
\end{proposition}
The score invariance follows because the phase features,
the KNN graph,
and the RBF edge encoding are invariant under a shared global
rigid transformation.
The anchor equivariance follows from $\sum_{i=1}^{N}w_i=1$.
A full proof is given in
Appendix~\ref{app:score_anchor_equivariance}.

For transformations that act only on the active region,
cross-region KNN edges change the graph structure,
so the anchor does not track them exactly.
PAF still provides two practical benefits.
First,
the learned score concentrates on the active region,
so the anchor is a stable reference point robust to
changes in phase-irrelevant geometry.
Second,
the score feature provides selective attention to task-relevant
geometry,
making the policy more sensitive to the active region
than to the rest of the scene.
We verify both properties through visualization and
quantitative analysis in Section~\ref{sec:ablations}.

\section{Experiments}
\label{sec:experiments}

\subsection{Simulation Experiments}
\label{sec:simulation_experiments}

\textbf{Simulation tasks.}
We evaluate on four multi-object manipulation tasks from
MimicGen~\citep{mandlekar2023mimicgen, mandlekar2021what},
built atop the Robosuite~\citep{zhu2025robosuitemodularsimulationframework} simulator:
Square,
Coffee,
Threading,
and Three-Piece Assembly,
each instantiated at the d2,
d3,
and d4 spatial randomization levels,
yielding twelve task-difficulty settings.
The d2 level applies $SE(2)$ randomization to the initial object
poses,
and d3 and d4 apply $SE(3)$ randomization,
with the randomization range increasing from d2 to d4.
Figure~\ref{fig:sim_tasks} shows the simulator
setup for each task and randomization level.
For each setting we use 100 demonstrations for training
and report the mean success rate over 50 trials and
3 seeds.

\begin{figure}[t]
    \centering
    \includegraphics[width=\linewidth]{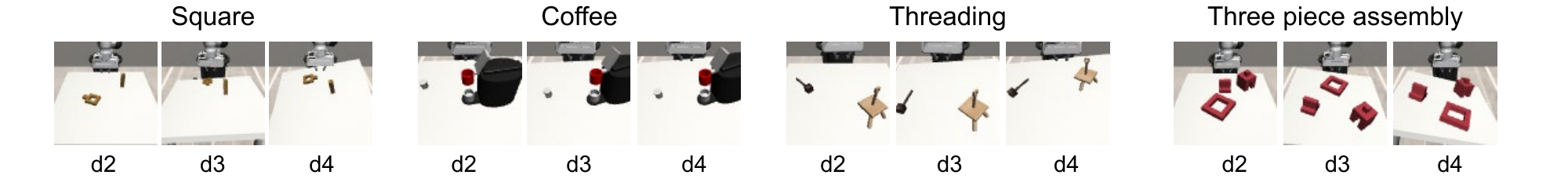}
    \caption{The four MimicGen tasks at the d2,
    d3,
    and d4 spatial randomization levels.}
    \label{fig:sim_tasks}
    \vspace{0.5em}
    \footnotesize
    \setlength{\tabcolsep}{4.3pt}
    \captionof{table}{Simulation success rate in percent on four
    MimicGen tasks.
    \emph{Equ.}~denotes the equivariance group of each policy,
    and DiffPo denotes Diffusion Policy.
    Bold marks the best entry in each column.
    Published SDP results are shown here. See
    Table~\ref{tab:sdp_reproduction} for reproduction and seed variability.}
    \label{tab:main_simulation}
    \begin{tabular}{l l ccc ccc ccc ccc cccc}
    \toprule
     & & \multicolumn{3}{c}{Coffee}
       & \multicolumn{3}{c}{Three Pc.}
       & \multicolumn{3}{c}{Square}
       & \multicolumn{3}{c}{Threading}
       & \multicolumn{4}{c}{Average} \\
    \cmidrule(lr){3-5}\cmidrule(lr){6-8}\cmidrule(lr){9-11}\cmidrule(lr){12-14}\cmidrule(lr){15-18}
    Method & Equ. &
    d2 & d3 & d4 &
    d2 & d3 & d4 &
    d2 & d3 & d4 &
    d2 & d3 & d4 &
    d2 & d3 & d4 & All \\
    \midrule
    SDP      & $SE(3)$   & 63 & 54 & 33 & 67 & 49 & 37 & 62 & 38 & 31 & 60 & 53 & 39 & 63.0 & 48.5 & 35.0 & 48.8 \\
    EquiDiff & $SO(2)$   & 65 & 43 & 29 & 37 & 15 & 8 & 39 & 3 & 3 & 39 & 20 & 10 & 45.0 & 20.3 & 12.5 & 25.9 \\
    DiffPo   & N/A       & 44 & 23 & 16 & 4 & 2 & 2 & 8 & 0 & 1 & 17 & 10 & 8 & 18.3 & 8.8 & 6.8 & 11.3 \\
    DP3      & N/A       & 34 & 20 & 0 & 0 & 0 & 0 & 7 & 0 & 13 & 12 & 4 & 1 & 13.3 & 6.0 & 3.5 & 7.6 \\
    \midrule
    PAF (Ours) & $SE(3)$   & \textbf{68} & \textbf{62} & \textbf{34} & \textbf{68} & \textbf{56} & \textbf{40} & \textbf{73} & \textbf{69} & \textbf{44} & \textbf{65} & \textbf{54} & \textbf{42} & \textbf{68.5} & \textbf{60.3} & \textbf{40.0} & \textbf{56.3} \\
    \bottomrule
    \end{tabular}
\end{figure}

\textbf{Baselines and implementation.}
We compare against
SDP~\citep{zhu2025se},
EquiDiff~\citep{wang2025equivariant},
DP3~\citep{ze2024dp3},
and Diffusion Policy~\citep{chi2025diffusion},
spanning $SE(3)$, $SO(2)$, and non-equivariant
architectures across point cloud, voxel, and RGB modalities.
Table~\ref{tab:main_simulation} uses published baseline
results~\citep{zhu2025se} except DP3 at d4, which is our implementation.
PAF uses the SDP authors' released demonstrations for d2 and
demonstrations regenerated with their official pipeline for d3 and d4.
We also reproduce SDP on PAF's exact demonstration sets under the
same evaluation protocol, sharing its $SO(3)$-equivariant diffusion
backbone to isolate the affordance-field contribution.
Appendix~\ref{app:sdp_reproduction} reports the reproduction and
standard deviations across seeds. Baseline and hyperparameter details
are in Appendices~\ref{app:baselines} and~\ref{app:hyperparameters}.

\textbf{Main results analysis.}
On the twelve task-difficulty settings reported in
Table~\ref{tab:main_simulation},
PAF raises the average success rate of SDP from
$48.8\%$ to $56.3\%$,
with the largest improvement at the intermediate d3 level,
where the average rises from $48.5\%$ to $60.3\%$.
This suggests that PAF is most useful when object poses vary in full
$SE(3)$ randomization but the task remains within a regime
where the underlying equivariant visuomotor policy can still
exploit the recovered geometric structure.
The other baselines drop substantially below SDP under d3 and d4
randomization,
confirming that SDP is the strongest existing comparison and the
gain we report is against a strong baseline.

\subsection{In-Depth Analysis}
\label{sec:ablations}

\begin{figure}[t]
    \centering
    \includegraphics[width=\linewidth]{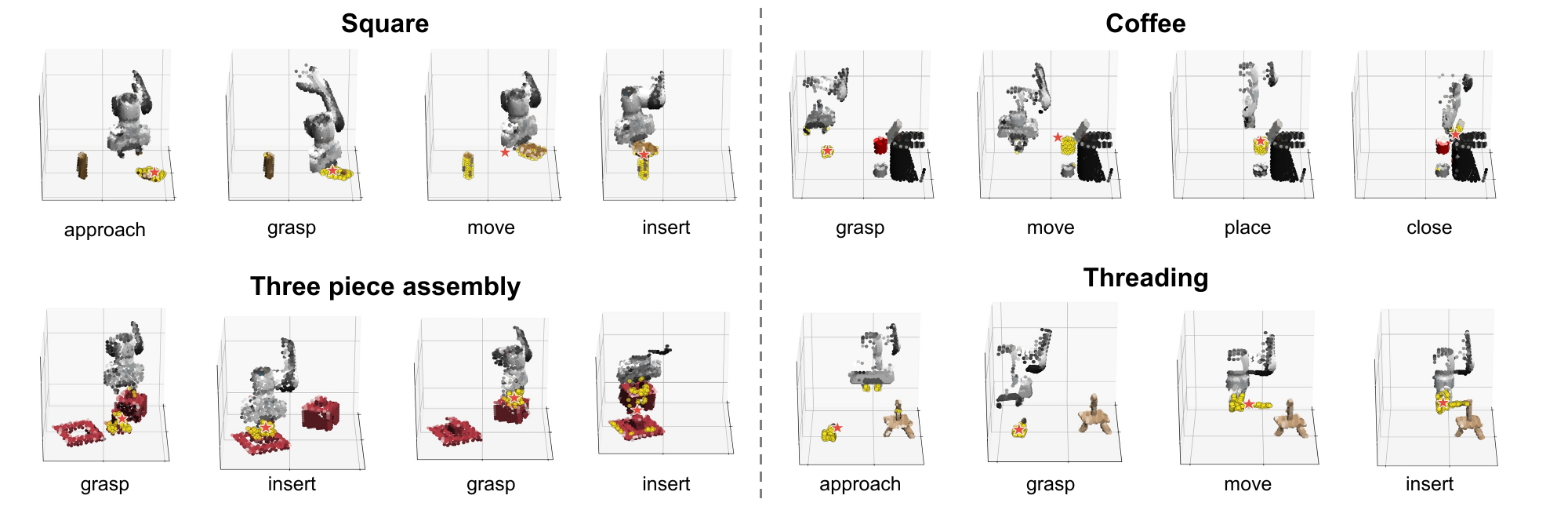}
    \caption{Affordance visualization on four MimicGen tasks.}
    \label{fig:sim_affordance_visualization}
    \vspace{0.5em}
    \begin{minipage}[t]{0.5\linewidth}
    \centering
    \scriptsize
    \setlength{\tabcolsep}{3.5pt}
    \captionof{table}{Component ablation on four task-difficulty
    settings.}
    \label{tab:ablation}
\begin{tabular}{l ccccc}
\toprule
Method     & Square & Coffee & Three Pc. & Square & Avg \\
           & d3     & d3     & d3        & d4     &     \\
\midrule
PAF (full) & \textbf{69} & \textbf{62} & \textbf{56} & \textbf{44} & \textbf{57.8} \\
\midrule
w/o anchor        & 43 & 46 & 40 & 33 & 40.5 \\
w/ keypose anchor & 35 & 51 & 48 & 32 & 41.5 \\
w/ hard scope     & 67 & 51 & 36 & 40 & 48.5 \\
w/o phase         & 60 & 55 & 52 & 38 & 51.3 \\
w/o routing       & 64 & 54 & 49 & 40 & 51.8 \\
w/o score         & 68 & 58 & 51 & 35 & 53.0 \\
\bottomrule
\end{tabular}
    \end{minipage}\hfill
    \begin{minipage}[t]{0.49\linewidth}
    \centering
    \scriptsize
    \setlength{\tabcolsep}{3pt}
    \renewcommand{\arraystretch}{1.39}
    \captionof{table}{Policy selectivity, anchor stability, and
    score change under local perturbations.}
    \label{tab:local_transform}
\begin{tabular}{l l cc c c}
\toprule
       & & \multicolumn{2}{c}{Selectivity ratio $\uparrow$}
       & Anchor & Score \\
\cmidrule(lr){3-4}
Perturb. & Task & PAF & w/o anchor
       & disp.\ (mm) & MAE \\
\midrule
\multirow{2}{*}{Trans.} & Square d3 & \textbf{1.41} & 1.29 & 0.05 & -- \\
                        & Coffee d3 & \textbf{1.57} & 1.31 & 0.00 & -- \\
\midrule
\multirow{2}{*}{Rot.}   & Square d3 & \textbf{1.22} & 1.02 & 0.03 & $2\times10^{-4}$ \\
                        & Coffee d3 & \textbf{1.35} & 1.00 & 0.00 & $2\times10^{-4}$ \\
\bottomrule
\end{tabular}
    \end{minipage}
\end{figure}

\textbf{Affordance visualization.}
Figure~\ref{fig:sim_affordance_visualization} highlights the
top-256 scored points in yellow and the anchor $c_t$ as a red
star.
The field concentrates on the object that governs the current
sub-phase rather than spreading over the whole scene.
On Square the score stays on the nut through approach and grasp,
then shifts to the peg for the move and insert sub-phases.
On Coffee it stays on the capsule at grasp,
hands off to the machine slot at move and place,
and to the lid at close,
tracking three distinct objects across four sub-phases.
Across all tasks the affordance field correctly identifies the
current manipulation phase from the phase context and predicts
the active region accordingly,
providing a local geometric inductive bias for the downstream
policy.

\textbf{Ablation studies.}
We ablate PAF on Square d3,
Coffee d3,
Three-Piece Assembly d3,
and Square d4.
Each row modifies PAF as described below.
\emph{w/o anchor} replaces $c_t$ with the gripper position $e_t$.
\emph{w/ keypose anchor} regresses $c_t$ directly to the next
keypose $y_t$ and removes the score channel,
slot routing,
and phase supervision.
\emph{w/ hard scope} keeps only the top-256 points by $\hat{s}_t$
instead of appending $\hat{s}_t$ as a feature.
\emph{w/o phase} removes the phase context encoder
$\Phi_{\mathrm{phase}}$ and with it the phase-to-slot router that
takes the phase context as input.
\emph{w/o routing} sets the slot router $\alpha_t$ to uniform.
\emph{w/o score} drops the score channel from $\tilde{O}_t$.

Table~\ref{tab:ablation} shows that the learned anchor is the
dominant component.
\emph{w/o anchor} loses $17.3$ points on average and falls below
SDP on Coffee~d3 and Three-Piece~d3,
and \emph{w/ keypose anchor} loses $16.3$ points,
so a directly regressed keypose anchor,
which also drops the score channel, routing, and phase supervision,
recovers none of the gain.
\emph{w/ hard scope} loses $9.3$ points,
concentrated on Coffee~d3 and Three-Piece~d3,
so the policy needs the full point cloud to track several objects
across phases.
\emph{w/o phase}, \emph{w/o routing}, and \emph{w/o score} lose
$6.5$, $6.0$, and $4.8$ points,
with routing mattering more on these two settings and the phase
context more on the Square settings.

\begin{figure}[t]
    \centering
    \includegraphics[width=\linewidth]{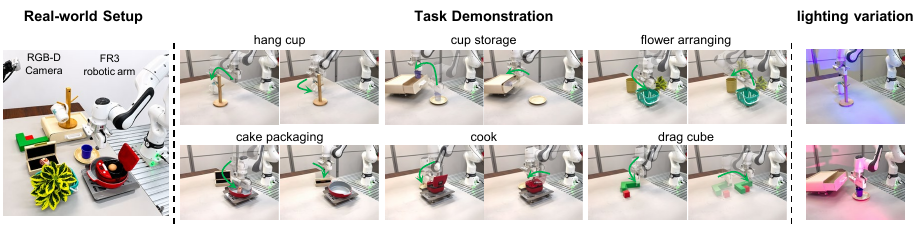}
    \caption{Real-robot setup and tasks.
    From left to right, a Franka Research 3 arm with an RGB-D
    camera, keyframes of the six tasks, and the lighting variation
    used to probe robustness.}
    \label{fig:real_setup}
    \vspace{0.5em}
    \footnotesize
    \setlength{\tabcolsep}{6pt}
    \captionof{table}{Real-robot success rate on six long-horizon
    tasks.
    N denotes normal lighting and L the out-of-distribution lighting
    variation.
    Bold marks the best entry in each column.}
    \label{tab:real_robot}
    \begin{tabular}{l cc cc cc cc cc cc cc}
    \toprule
       & \multicolumn{2}{c}{hang cup}
       & \multicolumn{2}{c}{cup storage}
       & \multicolumn{2}{c}{flower arr.}
       & \multicolumn{2}{c}{cake pkg.}
       & \multicolumn{2}{c}{cook}
       & \multicolumn{2}{c}{drag cube}
       & \multicolumn{2}{c}{Avg} \\
    \cmidrule(lr){2-3}\cmidrule(lr){4-5}\cmidrule(lr){6-7}\cmidrule(lr){8-9}\cmidrule(lr){10-11}\cmidrule(lr){12-13}\cmidrule(lr){14-15}
    Method     & N & L & N & L & N & L & N & L & N & L & N & L & N & L \\
    \midrule
    SDP        & 65 & 53 & 45 & 38 & 60 & 48 & 63 & 53 & 40 & 28 & 75 & 75 & 58.0 & 49.2 \\
    DP3        & 63 & 53 & 40 & 28 & 53 & 40 & 50 & 33 & 43 & 28 & 65 & 60 & 52.3 & 40.3 \\
    \midrule
    PAF (Ours) & \textbf{78} & \textbf{68} & \textbf{55} & \textbf{43} & \textbf{68} & \textbf{65} & \textbf{70} & \textbf{65} & \textbf{48} & \textbf{38} & \textbf{83} & \textbf{80} & \textbf{67.0} & \textbf{59.8} \\
    \bottomrule
    \end{tabular}
    \par\vspace{1.2em}
    \includegraphics[width=0.92\linewidth]{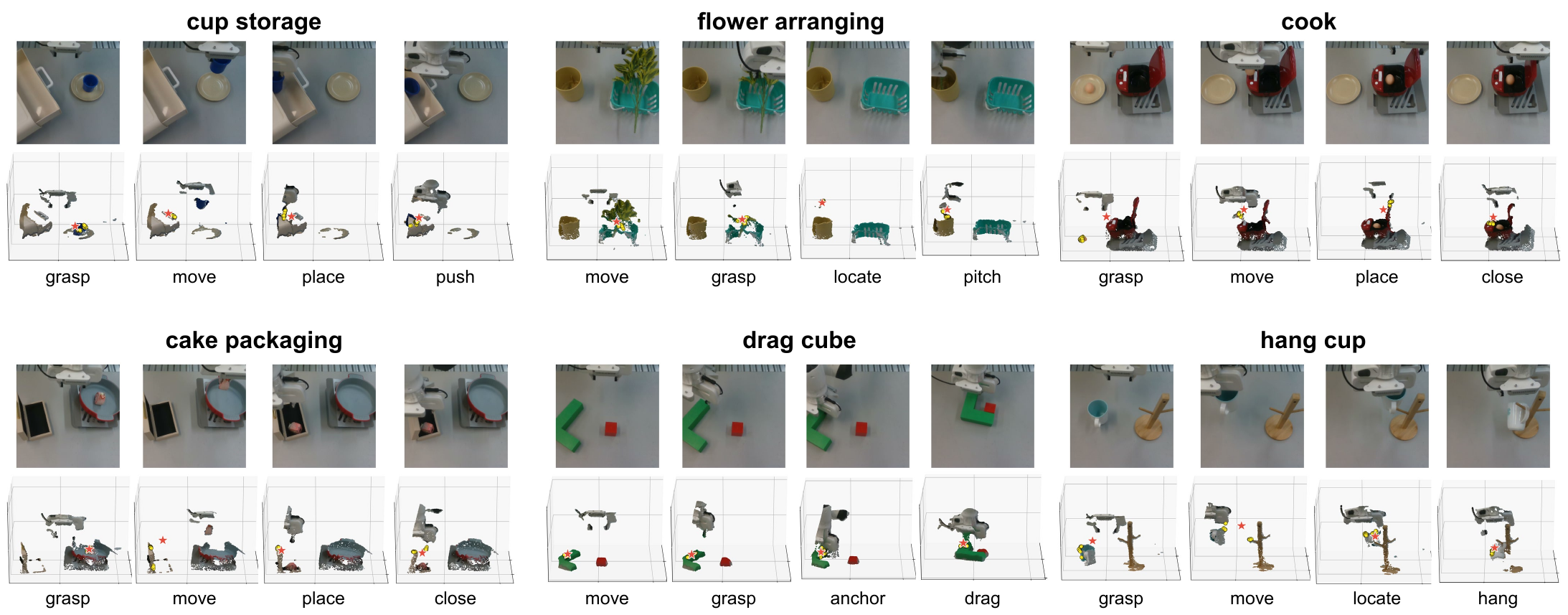}
    \par
    \captionof{figure}{Phase-conditioned affordance field on the six
    real-robot tasks.
    For each task the top row shows the RGB keyframe per sub-phase
    and the bottom row shows the point cloud with the affordance
    field,
    where the highlighted region marks the top-256 points by
    predicted score.}
    \label{fig:real_affordance}
\end{figure}

\textbf{Policy selectivity and anchor stability.}
To verify the two properties of
Section~\ref{sec:method_theoretical_analysis},
we measure how the policy responds to localized translations and
rotations of the point cloud.
The active subset is the top-256 points by anchor weight,
and the inactive subset is low-score points at least 10~cm from
the anchor.
Each subset is translated by 1/3/5~cm along $\pm x$/$\pm y$/$\pm z$
or rotated about the anchor by 5/10/15 degrees about
$\pm x$/$\pm y$/$\pm z$,
giving 18 perturbations of each kind.
We average the action change over them and the absolute change of
$\hat{s}_t$ on the rotated active points.

Table~\ref{tab:local_transform} shows that PAF responds more
selectively to the active region than the w/o-anchor variant.
The selectivity ratio,
the action change under active perturbations divided by that
under inactive ones,
is 1.41 vs 1.29 on Square and 1.57 vs 1.31 on Coffee under
translation,
and 1.22 vs 1.02 and 1.35 vs 1.00 under rotation,
where the w/o-anchor variant no longer separates the two regions.
Under inactive perturbations the anchor moves by at most 0.05~mm
against translations of 30~mm on average and by at most 0.03~mm
under rotation,
and under active rotations the affordance score changes by about
$2\times10^{-4}$,
so the anchor ignores phase-irrelevant geometry and the score is
stable under rotation of the region it attends to.
These measurements characterize selectivity and anchor stability
rather than exact local equivariance.

\subsection{Real-Robot Experiments}
\label{sec:real_robot_experiments}

\textbf{Real-world settings.}
We design six multi-object, long-horizon manipulation tasks,
with planar or $SE(3)$ object pose variation applied to four of
them across episodes.
Figure~\ref{fig:real_setup} shows the robot platform and task
examples.
Each task requires the active object to switch across two or
three phases of the manipulation.
Point clouds are obtained by back-projection of a single RGB-D
image without multi-view fusion.
Per task,
we train a single policy on 50 demonstrations collected under
normal lighting only,
and evaluate under both normal lighting (N) and
an out-of-distribution (OOD) lighting variation (L) that applies
colored illumination,
disturbing the RGB channel without changing the geometry.
For each condition we run two rollouts of 20 trials and report
their mean success rate.
Per-task descriptions,
pose ranges,
slot settings,
and rollout keyframes are in Appendix~\ref{app:real_setup_rollouts}.

\textbf{Real-robot results.}
Table~\ref{tab:real_robot} compares PAF against
SDP~\citep{zhu2025se} and DP3~\citep{ze2024dp3},
both of which share the point cloud input modality with PAF.
PAF wins every task under both lighting conditions.
Under normal lighting (N),
PAF averages $67.0\%$ success vs SDP $58.0\%$ and DP3 $52.3\%$.
Under the out-of-distribution lighting variation (L),
the averages become $59.8\%$, $49.2\%$, and $40.3\%$.
PAF thus remains above both baselines under OOD lighting,
and its lead over SDP on the 240 OOD-lighting trials is
significant at $p=0.02$ under a two-sided Fisher exact test.
Compared to SDP,
PAF rarely fails in early approach and grasp phases,
and its remaining failures concentrate on later contact-rich
interactions such as insertion and tool use.
A side-by-side failure comparison between PAF and SDP is
provided in Appendix~\ref{app:failure_cases},
and inference overhead is reported in
Appendix~\ref{app:inference_overhead}.

\textbf{Affordance field visualization.}
Figure~\ref{fig:real_affordance} shows that on the real robot the
field concentrates on the object that governs the current
sub-phase rather than spreading over the whole scene,
mirroring the behavior observed in simulation in
Figure~\ref{fig:sim_affordance_visualization}.
On cup storage the score focuses on the cup for \emph{grasp},
shifts to the drawer for \emph{move},
returns to the cup for \emph{place},
and moves to the drawer front for \emph{push}.
On cake packaging the score tracks the cake during
\emph{grasp},
shifts to the box for \emph{move},
and switches to the lid after \emph{place}.
On hang cup the score concentrates on the cup during
\emph{grasp} and shifts to both the cup and the rack during
\emph{move} and \emph{locate}.
The phase names are author annotations,
and Appendix~\ref{app:real_setup_rollouts} gives their mapping to
the gripper-derived sub-phase labels.

\section{Conclusion}
\label{sec:conclusion_and_limitations}

We presented PAF,
a phase-conditioned affordance field that augments an
$SE(3)$-equivariant visuomotor policy with a local geometric
inductive bias.
PAF scores every observed point by its relevance to the current
manipulation phase,
recenters the policy on the score-weighted centroid,
and feeds the scores back as a feature,
so the frame origin follows the active region while global
$SE(3)$ equivariance is preserved.
PAF raises the average success rate of SDP from $48.8\%$ to
$56.3\%$ on twelve MimicGen settings and leads SDP and DP3 on six
real-robot tasks under both lighting conditions,
and the simulation ablations attribute most of the gain to the
learned anchor.

\section{Limitations}
\label{sec:limitations}

PAF has three limitations that point to future work.
First,
it provides a learned inductive bias rather than strict local
$SE(3)$-equivariance,
since score invariance and anchor equivariance hold provably
only under a shared global transformation.
Second,
the affordance field is supervised by pseudo-labels derived from
future end-effector positions,
which may not capture the fine geometry of contact-rich steps.
Third,
the slot count is currently tied to the grasp-release structure
of a task.
We plan to pursue tighter guarantees on local
$SE(3)$-equivariance,
richer supervision signals for the affordance field,
and more flexible slot formulations that adapt to diverse task
structures.

\clearpage
\acknowledgments{
This work was supported in part by the Fundamental and Interdisciplinary Disciplines Breakthrough Plan of the Ministry of Education of China under Grant No.~JYB2025XDXM210, the NSFC under Grant Nos.~62125305, U23A20339, 62573339, and 62503380, the Natural Science Foundation of Shaanxi Province under Grant No.~\mbox{2025SYSSYSZD-083}, the State Grid Corporation of China Science and Technology Projects \mbox{52060025005B-439-ZN}, the Shaanxi Provincial Key R\&D Program under Grant No.~\mbox{2026JC-ZDGG-17}, and the CIE-Tencent Robotics X Rhino-Bird Focused Research Program.
}

\bibliography{references}

\clearpage
\appendix

\section{Additional Method Details}
\label{app:method_details}

\subsection{Multi-Object Equivariance Derivations}
\label{app:multi_object_derivations}

\textbf{Translation invariance under global canonicalization.}
Let $\bar{x}_i=x_i-e_t$ and $\bar{e}_{\tau}^{a}=e_{\tau}^{a}-e_t$
denote canonicalized point and action positions around the global
reference $e_t$,
typically the current end-effector position.
Under a global rigid transformation $g=(R,u)\in SE(3)$,
both the points and the reference translate together,
so the canonicalized coordinates transform purely by rotation,
\begin{equation}
    \bar{x}'_i
    =
    (Rx_i+u)-(Re_t+u)
    =
    R\bar{x}_i,
    \qquad
    \bar{e}'^{a}_{\tau}
    =
    R\bar{e}^{a}_{\tau},
    \label{eq:canonicalization_identity}
\end{equation}
which is then handled by $SO(2)$- or $SO(3)$-equivariant networks.

\textbf{Independent object translation residual under global canonicalization.}
The same canonicalization reference $e_t$ tracks the end-effector,
which need not move with all task objects.
With object-wise translations $u_j$ on the points of object $j$
and reference translation $u_e$ on $e_t$,
the canonicalized coordinate of point $i$ on object $j$ becomes
\begin{equation}
    \bar{x}_i^{j\prime}
    =
    (x_i^j+u_j)-(e_t+u_e)
    =
    \bar{x}_i^j+(u_j-u_e).
    \label{eq:translation_canonicalization_residual}
\end{equation}
A single global reference therefore leaves residual offsets
$u_j-u_e$ whenever objects translate independently of the
reference.

\textbf{Rotation equivariance mismatch.}
After translation is factored out,
a whole-scene equivariant network expects the canonicalized scene
to rotate by one shared rotation,
whereas object-wise scene variation gives independent rotations,
\begin{equation}
    \bar{O}^{\prime,\mathrm{global}}_t
    =
    \bigcup_{j=1}^{J} R\bar{O}_t^{j}
    =
    R\bar{O}_t,
    \qquad
    \bar{O}^{\prime,\mathrm{obj}}_t
    =
    \bigcup_{j=1}^{J} R_j\bar{O}_t^{j}
    \not\equiv
    R\bar{O}_t
    \quad\text{in general}.
    \label{eq:global_vs_objectwise_rotation}
\end{equation}
The first case admits one well-defined group action on the whole
scene.
The second is not generated by any single group element,
so a global rotation-equivariance law has no scene-level target
transformation.

Eq.~\eqref{eq:translation_canonicalization_residual} shows that
the fixed gripper reference leaves residual offsets
$u_j-u_e$ under independent object motion.
PAF instead learns a phase-conditioned reference,
the anchor $c_t$,
that recenters on whichever region governs the current phase.

\subsection{Definitions of Phase Context Vector}
\label{app:phase_feature_computation}

The phase context vector $\xi_t$ consists of six rigid-motion
invariant statistics computed from the recent robot state history
$H_t$,
namely end-effector speed $v_\tau$,
speed change $\Delta v_\tau$,
turning angle $\gamma_\tau$,
relative orientation angular displacement $\omega_\tau$,
gripper state $q_\tau$,
and gripper change $\Delta q_\tau$.
Their definitions follow.
For consecutive history steps $\tau=t-L_h+2,\ldots,t$,
we compute the end-effector displacement,
speed magnitude,
and gripper change as
\begin{equation}
    \Delta e_{\tau}=e_{\tau}-e_{\tau-1},
    \qquad
    v_{\tau}=\|\Delta e_{\tau}\|,
    \qquad
    \Delta q_{\tau}=q_{\tau}-q_{\tau-1}.
    \label{eq:appendix_motion_position_gripper}
\end{equation}
For triples of positions $\tau=t-L_h+3,\ldots,t$,
we compute the speed change
$\Delta v_{\tau}=v_{\tau}-v_{\tau-1}$
and the turning angle
\begin{equation}
    \gamma_{\tau}
    =
    \arccos
    \left(
    \operatorname{clip}\!\left(
    \frac{
        \Delta e_{\tau-1}^{\top}\Delta e_{\tau}
    }{
        \|\Delta e_{\tau-1}\|\|\Delta e_{\tau}\|+\epsilon
    },\,
    -1+\epsilon,\,
    1-\epsilon
    \right)
    \right),
    \label{eq:appendix_motion_turning_angle}
\end{equation}
where $\epsilon$ is a small numerical constant and
$\operatorname{clip}$ guards the $\arccos$ domain against numerical
overshoot and degenerate zero-displacement steps.
For orientation over $\tau=t-L_h+2,\ldots,t$,
in our implementation,
the orientation representation $r_{\tau}$ is converted to a unit
quaternion,
and we compute the relative quaternion
$\delta r_{\tau}=r_{\tau}^{-1}\otimes r_{\tau-1}$,
where $\otimes$ denotes quaternion multiplication.
Using $r_{\tau}$ here to denote the converted quaternion,
let $\delta r_{\tau,w}$ and $\delta r_{\tau,xyz}$ denote the scalar and
vector parts of $\delta r_{\tau}$.
The angular displacement is
\begin{equation}
    \omega_{\tau}
    =
    2\,\operatorname{atan2}
    \left(
        \|\delta r_{\tau,xyz}\|,
        |\delta r_{\tau,w}|
    \right),
    \label{eq:appendix_motion_angular_displacement}
\end{equation}
where $\operatorname{atan2}(y,x)$ is the two-argument arctangent and the
factor $2$ converts the half-angle represented by a unit quaternion into
a rotation angle.
Each statistic is invariant under shared global $SE(3)$ transforms
because it depends only on displacement norms,
angular differences,
or scalar gripper values.

\subsection{Graph Construction Details}
\label{app:graph_construction}

\textbf{KNN graph and RBF edges.}
Let $p_i$ denote the coordinate of a graph node,
either a point in $O_t$ or a recent end-effector in $H_t$.
The neighbor set $\mathcal{N}(i)$ is obtained by
$k$-nearest neighbors under
$d_{ij}=\|p_i-p_j\|$.
Each distance is encoded by an RBF basis,
\begin{equation}
    \rho_{ij,m}
    =
    \exp\!\left(
        -\left(\frac{d_{ij}-\mu_m}{\sigma_{\rho}}\right)^{2}\right),
    \quad m=1,\ldots,M,
    \label{eq:rbf_distance_encoding}
\end{equation}
where $M$ is the number of RBF kernels,
$d_{\max}$ is the maximum encoding distance,
the centers $\{\mu_m\}_{m=1}^{M}$ are uniformly spaced on
$[0,d_{\max}]$,
$\sigma_{\rho}=d_{\max}/(M-1)$,
and $\rho_{ij}=(\rho_{ij,1},\ldots,\rho_{ij,M})\in\mathbb{R}^{M}$
is the RBF edge feature.

\textbf{Node initialization.}
Let $\tilde{h}_i^0\in\mathbb{R}^{d_h}$ be the initial feature of
node $i$,
obtained from its non-geometric attributes
(e.g., color for point nodes, gripper state for end-effector nodes).
We inject the phase context by broadcast addition followed by
layer normalization,
\begin{equation}
    h_i^0
    =
    \mathrm{LN}\!\left(
        \tilde{h}_i^0+P_{\mathrm{phase}}h_t^{\mathrm{phase}}\right),
    \label{eq:node_init}
\end{equation}
where $\mathrm{LN}$ denotes layer normalization,
$P_{\mathrm{phase}}\in\mathbb{R}^{d_h\times d_{\mathrm{phase}}}$
projects the phase context to the node hidden dimension $d_h$,
and the same projected vector is added to all graph nodes.

\textbf{Score and slot heads.}
After the final message-passing layer,
two independent MLP heads decode each point-node feature
$h_i^{L_g}$,
\begin{equation}
    s_i=\phi_{\mathrm{score}}(h_i^{L_g})\in\mathbb{R},
    \qquad
    m_i=\phi_{\mathrm{slot}}(h_i^{L_g})\in\mathbb{R}^{K},
    \label{eq:score_slot_heads}
\end{equation}
where $L_g$ is the number of graph layers,
$\phi_{\mathrm{score}}$ predicts the base score logit,
and $\phi_{\mathrm{slot}}$ predicts the $K$ slot membership logits.

\subsection{Proof of Score Invariance and Anchor Equivariance}
\label{app:score_anchor_equivariance}

We restate and prove Proposition~1 from
Section~\ref{sec:method_theoretical_analysis} under a shared global
$SE(3)$ transformation.
It does not claim invariance to arbitrary object-wise
transformations,
which can change cross-object distances and the whole-scene
graph.

\begin{propositionrestate}
Let $g=(R,u)\in SE(3)$ act on all observed point and end-effector
poses,
and let $\hat{s}_t^{\,\prime}(i)$ and $c_t^{\prime}$ denote the
affordance score and anchor computed from the transformed
observation $g\cdot S_t$.
Then $\hat{s}_t^{\,\prime}(i)=\hat{s}_t(i)$ (score invariance)
and $c_t^{\prime}=Rc_t+u$ (anchor equivariance).
\end{propositionrestate}

\noindent\textbf{Proof.}
For any pair of graph nodes,
\begin{equation}
    d_{ij}^{\prime}
    =
    \|(Rp_i+u)-(Rp_j+u)\|
    =
    \|R(p_i-p_j)\|
    =
    \|p_i-p_j\|
    =
    d_{ij}.
    \label{eq:appendix_distance_invariance}
\end{equation}
Therefore the KNN graph and the RBF edge encoding are unchanged.
Since each message-passing layer uses only node features and the
RBF-encoded distances,
the point-node features after every layer are unchanged by the global
transformation.

The phase context is also invariant.
End-effector speed magnitudes,
speed changes,
turning angles,
relative quaternion angular displacements,
gripper states,
and gripper deltas are unchanged by applying the same rigid
transform to the whole scene and robot trajectory.
The phase context vector is therefore invariant,
\begin{equation}
    \xi_t(g\cdot H_t)=\xi_t(H_t),
    \label{eq:appendix_motion_invariance}
\end{equation}
and the phase context
$h_t^{\mathrm{phase}}=\Phi_{\mathrm{phase}}(\xi_t)$ is invariant as
well.
The score head,
slot head,
and phase-to-slot router are functions of invariant node features and
invariant phase context.
Hence
\begin{equation}
    b_i^{\prime}=b_i,
    \qquad
    \beta_{i,k}^{\prime}=\beta_{i,k},
    \qquad
    \alpha_{t,k}^{\prime}=\alpha_{t,k},
    \qquad
    \hat{s}_t^{\,\prime}(i)=\hat{s}_t(i).
    \label{eq:appendix_score_components_invariant}
\end{equation}

The anchor weights are functions only of the invariant scores,
so $w_i^{\prime}=w_i$.
The transformed anchor is therefore
\begin{equation}
    c_t^{\prime}
    =
    \sum_i w_i^{\prime}(Rx_i+u)
    =
    R\sum_i w_i x_i
    +
    u\sum_i w_i
    =
    Rc_t+u,
    \label{eq:appendix_anchor_proof}
\end{equation}
because $\sum_i w_i=1$.
This proves score invariance and anchor equivariance under a shared
global $SE(3)$ transformation.
Hence the PAF policy of Eq.~\eqref{eq:paf_policy_structure} is
globally $SE(3)$-equivariant.
Under $g$,
$\tilde{x}_i^{\prime}=R\tilde{x}_i$,
$\tilde{e}_{\tau}^{\prime}=R\tilde{e}_{\tau}$,
and $r_{\tau}^{\prime}=Rr_{\tau}$,
while the score channel,
point features,
and gripper states are unchanged,
so the $SO(3)$-equivariant backbone rotates its predicted actions
by $R$,
and shifting positions back by $c_t^{\prime}=Rc_t+u$ gives
Eq.~\eqref{eq:policy_equivariance}.

\subsection{Phase Label and Slot Routing}
\label{app:slot_warmup}

\textbf{Pseudo-label formulation.}
The target action position $y_t$ is the expert end-effector
position at the first gripper open-close transition within the
future action horizon.
When no transition occurs within the horizon,
$y_t$ falls back to the end-effector position at the horizon
end.
The pseudo-label and end-effector repulsion are
\begin{equation}
    \ell_i
    =
    \exp(-\|x_i-y_t\|/\tau_y),
    \quad
    \tilde{\ell}_i
    =
    \ell_i\!\left[
        1-(1-\eta)\exp(-\|x_i-e_t\|/\tau_{\mathrm{ee}})\right],
    \label{eq:action_pseudo_label}
\end{equation}
where $\tau_y$ and $\tau_{\mathrm{ee}}$ are spatial scales and
$\eta\in[0,1)$ is a minimum-scale floor,
so the pseudo-label is damped near the current end-effector
but never below a fraction $\eta$ of $\ell_i$.
The score loss is binary cross entropy on logits,
\begin{equation}
    \mathcal{L}_{\mathrm{score}}
    =
    \frac{1}{N}\sum_{i=1}^{N}\mathrm{BCE}(s_i,\tilde{\ell}_i).
    \label{eq:score_loss}
\end{equation}

\textbf{Phase and slot labels.}
The sub-phase label $\varphi_t\in\{0,1,2,3\}$ is derived from
the gripper state at the start and end of the future action
horizon.
$\varphi_t=0$ when the gripper remains open across the horizon,
$\varphi_t=1$ when it transitions from open to closed,
$\varphi_t=2$ when it remains closed,
and $\varphi_t=3$ when it transitions from closed to open.
For multi-cycle tasks,
$\kappa_t\in\{0,\ldots,\text{max\_cycles}-1\}$ is the number of
gripper closures observed in the episode before the current step,
clipped at $\text{max\_cycles}-1$.
A separate phase classifier head maps $h_t^{\mathrm{phase}}$ to
sub-phase and cycle predictions.
The phase loss $\mathcal{L}_{\mathrm{phase}}$ is the sum of
cross-entropy on the sub-phase prediction against $\varphi_t$
and cross-entropy on the cycle prediction against $\kappa_t$.
The slot label combines the cycle index with the open/closed
gripper state,
\begin{equation}
    z_t
    =
    2\kappa_t
    +
    \mathbf{1}[\varphi_t\ge 2],
    \qquad
    z_t\in\{0,\ldots,K-1\},
    \label{eq:appendix_slot_label}
\end{equation}
yielding an open-gripper and a closed-gripper slot per value
of $\kappa_t$ and $K=2\times\text{max\_cycles}$ slot classes.
The slot loss $\mathcal{L}_{\mathrm{slot}}$ is cross-entropy on
the phase-to-slot router output $\alpha_t$ against $z_t$.

\textbf{Gradient routing of slot membership.}
The per-point slot membership logits $m_i$ are not directly
supervised.
They receive gradient only through the anchor $c_t$ and the
score-as-feature input to the diffusion policy under
$\mathcal{L}_{\mathrm{ddpm}}$,
so $m_i$ is shaped by what the policy needs from
$\hat{s}_t$ and $c_t$ to predict the action,
not by a hand-specified per-point segmentation target.

\section{Simulation Experiment Details}
\label{app:sim_details}

\subsection{MimicGen Tasks}
\label{app:sim_tasks_visualization}

We evaluate on four multi-object MimicGen tasks.
\emph{Square} inserts a square nut onto a peg.
\emph{Coffee} places a capsule into a coffee machine and closes
the lid.
\emph{Threading} threads a needle through a target hole.
\emph{Three-Piece Assembly} sequentially assembles three parts,
involving two separate grasp-and-place cycles.
Each task is instantiated at three spatial randomization levels,
d2 with $SE(2)$ initial pose randomization,
and d3 and d4 with $SE(3)$ randomization at increasing range.
Figure~\ref{fig:sim_tasks} shows the simulator setup for these tasks
and randomization levels.

\subsection{Baselines}
\label{app:baselines}

We compare against four visuomotor policy baselines.
SDP~\citep{zhu2025se} is an $SE(3)$-equivariant point cloud
diffusion policy that achieves translational invariance via
gripper-frame canonicalization and rotational equivariance via
spherical Fourier features.
EquiDiff~\citep{wang2025equivariant} applies discrete
$SO(2)$-equivariance to a voxel diffusion policy in an
end-to-end manner.
Diffusion Policy~\citep{chi2025diffusion} is a 2D RGB diffusion
policy without equivariance.
DP3~\citep{ze2024dp3} is a point cloud diffusion policy without
equivariance.
We omit three additional baselines that are not directly
comparable under our protocol.
EquiBot~\citep{yang2025equibot} requires per-task object
segmentation,
which does not generalize to multi-object tasks where the active
object changes across phases.
AnchorDP3~\citep{zhao2025anchordp33daffordanceguided} and AffordDP~\citep{wu2025afforddp} are evaluated under
different settings (dual-arm sparse-keypose prediction and
cross-category zero-shot transfer respectively) and depend on
segmentation masks or foundation-model-estimated affordances,
which are unavailable in our self-supervised, single-arm
protocol.
As an in-setting substitute for these affordance-augmented
diffusion baselines,
the \emph{w/o phase} ablation in Table~\ref{tab:ablation}
removes phase conditioning and reduces the affordance predictor
to a single phase-agnostic score,
approximating the phase-agnostic affordance prior of those methods
while keeping the
data, backbone, and evaluation protocol identical to PAF.

\subsection{SDP Reproduction and Seed Variability}
\label{app:sdp_reproduction}

To check the effect of demonstration provenance, we reproduce SDP
on the exact demonstration sets used to train PAF.
The d2 sets are the SDP authors' released files, whereas the
unreleased d3 and d4 sets are regenerated using their official pipeline.
Table~\ref{tab:sdp_reproduction} reports the mean success rate and
standard deviation across three seeds for PAF and the reproduced SDP,
with 50 evaluation trials per seed and setting.
The published SDP results are included separately for reference.

PAF outperforms reproduced SDP across all twelve settings in
Table~\ref{tab:sdp_reproduction}.
Averaging the displayed per-setting means gives approximately
$56.3\%$ for PAF and $47.6\%$ for reproduced SDP.
The published SDP average is $48.8\%$.
All twelve reproduced SDP results differ from the published values
by at most 3 percentage points.
For Square d3, reproduced SDP achieves $39\%$, compared with the
published result of $38\%$.
These results support the performance advantage of PAF over SDP
when the training demonstrations are held fixed.

\begin{table}[htbp]
\centering
\small
\setlength{\tabcolsep}{6pt}
\caption{Success rate in percent on identical training
demonstrations for PAF and reproduced SDP.
Higher is better.
Entries are mean $\pm$ standard deviation across three seeds.
Published SDP entries are the means retained in
Table~\ref{tab:main_simulation}.}
\label{tab:sdp_reproduction}
\begin{tabular}{l c c c}
\toprule
Setting & PAF & SDP (reproduced) & SDP (published) \\
\midrule
Coffee d2 & $68\pm0.0$ & $62\pm2.0$ & $63$ \\
Coffee d3 & $62\pm2.0$ & $53\pm2.3$ & $54$ \\
Coffee d4 & $34\pm0.0$ & $31\pm2.3$ & $33$ \\
\midrule
Three-Piece d2 & $68\pm4.0$ & $64\pm3.5$ & $67$ \\
Three-Piece d3 & $56\pm0.0$ & $48\pm2.0$ & $49$ \\
Three-Piece d4 & $40\pm2.0$ & $37\pm2.3$ & $37$ \\
\midrule
Square d2 & $73\pm2.3$ & $61\pm3.1$ & $62$ \\
Square d3 & $69\pm2.3$ & $39\pm3.1$ & $38$ \\
Square d4 & $44\pm5.3$ & $29\pm1.2$ & $31$ \\
\midrule
Threading d2 & $65\pm1.2$ & $58\pm2.0$ & $60$ \\
Threading d3 & $54\pm0.0$ & $51\pm1.2$ & $53$ \\
Threading d4 & $42\pm2.0$ & $38\pm0.0$ & $39$ \\
\bottomrule
\end{tabular}
\end{table}

\subsection{Implementation Details}
\label{app:hyperparameters}

Table~\ref{tab:appendix_hyperparameters} lists the
hyperparameters used in the main simulation results.
The slot count $K$ is task-dependent and equals
$2\times\text{max\_cycles}$,
where $\text{max\_cycles}\!=\!1$ for the single-cycle tasks
Square,
Coffee,
and Threading,
and $\text{max\_cycles}\!=\!2$ for the two-cycle task Three-Piece
Assembly.
The diffusion policy backbone matches Spherical Diffusion
Policy~\citep{zhu2025se},
which is also the backbone of the reproduced SDP baseline in
Table~\ref{tab:sdp_reproduction},
so this same-data comparison isolates the contribution
of the phase-conditioned affordance field.
We optimize with AdamW~\citep{loshchilov2018decoupled}
and apply EMA weight averaging following the diffusion-policy
backbone.
The diffusion policy backbone parameters listed below are
inherited from SDP and included for
reproducibility.

\begin{table}[h]
\centering
\small
\setlength{\tabcolsep}{6pt}
\caption{Hyperparameters used in the main simulation results.
Values correspond to the PAF configuration used in
Table~\ref{tab:main_simulation}.}
\label{tab:appendix_hyperparameters}
\begin{tabular}{l l}
\toprule
Parameter & Value \\
\midrule
\multicolumn{2}{l}{\emph{Loss weights}} \\
Score loss weight $\lambda_s$ & $1.0$ \\
Phase loss weight $\lambda_p$ & $0.2$ \\
Slot loss weight $\lambda_z$ & $0.1$ \\
\midrule
\multicolumn{2}{l}{\emph{Pseudo-label}} \\
Pseudo-label scale $\tau_y$ & $0.03$ \\
End-effector repulsion scale $\tau_{\mathrm{ee}}$ & $0.12$ \\
Repulsion floor $\eta$ & $0.5$ \\
\midrule
\multicolumn{2}{l}{\emph{Slot routing}} \\
Total slot count $K$ & $2$ for single-cycle tasks, $4$ for Three-Piece Assembly \\
\midrule
\multicolumn{2}{l}{\emph{Affordance graph}} \\
KNN neighbors $k_{\mathrm{nn}}$ & $24$ \\
Graph layers $L_g$ & $2$ \\
Hidden dim $d_h$ & $64$ \\
RBF basis count $M$ & $16$ \\
RBF range $d_{\max}$ on normalized coordinates & $0.125$ \\
Motion history $L_h$ & $16$ \\
Phase context dim $d_{\mathrm{phase}}$ & $128$ \\
\midrule
\multicolumn{2}{l}{\emph{Diffusion policy backbone}} \\
Encoder & 5-layer $SO(3)$-equivariant EquiformerV2 \\
Sphere channels & $\{16, 32, 64, 128, 256\}$ \\
Maximum degree $\ell_{\max}=m_{\max}$ & $2$ \\
Action horizon $T_a$ & $16$ \\
Observation steps $T_o$ & $2$ \\
Action steps & $8$ \\
Noise schedule & DDPM, $100$ timesteps \\
Inference steps & $100$ (simulation, DDPM~\citep{NEURIPS2020_4c5bcfec}), $20$ (real robot, DDIM~\citep{song2021denoising}) \\
Score-as-feature channel & enabled \\
\midrule
\multicolumn{2}{l}{\emph{Optimization}} \\
Optimizer & AdamW \\
Learning rate & $1\times 10^{-4}$ \\
Betas & $(0.95, 0.999)$ \\
Weight decay & $1\times 10^{-6}$ \\
LR schedule & cosine, $500$ warmup steps \\
Batch size & $32$ \\
Epochs at $n_{\mathrm{demo}}=100$ & $500$ \\
EMA power, max & $0.75$, $0.9999$ \\
Hardware & $8\times$ RTX 3090, one training run per GPU \\
\bottomrule
\end{tabular}
\end{table}

\section{Real-Robot Experiment Details}
\label{app:real_robot_task_details}

\subsection{Setup and Rollouts}
\label{app:real_setup_rollouts}

\textbf{Platform.}
The robot is a Franka Research 3 arm with a parallel-jaw gripper.
Scene observation comes from a single RealSense D435 RGB-D camera
mounted in a third-person view of the workspace,
and point clouds are obtained by direct back-projection of the
depth image without multi-view fusion.

\begin{figure}[!t]
    \centering
    \includegraphics[width=0.85\linewidth]{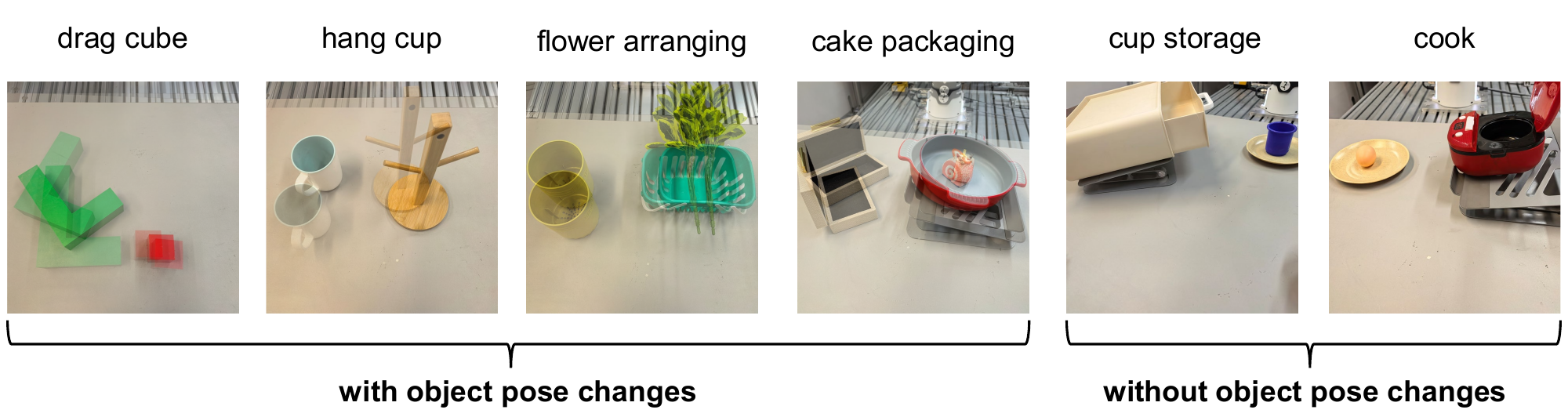}
    \caption{Real-robot task settings.
    The first four tasks involve object pose changes across
    episodes, and the last two have fixed object layouts.}
    \label{fig:task_settings}
    \vspace{0.8em}
    \centering
    \begin{subfigure}[t]{0.46\linewidth}
        \centering
        \includegraphics[width=\linewidth]{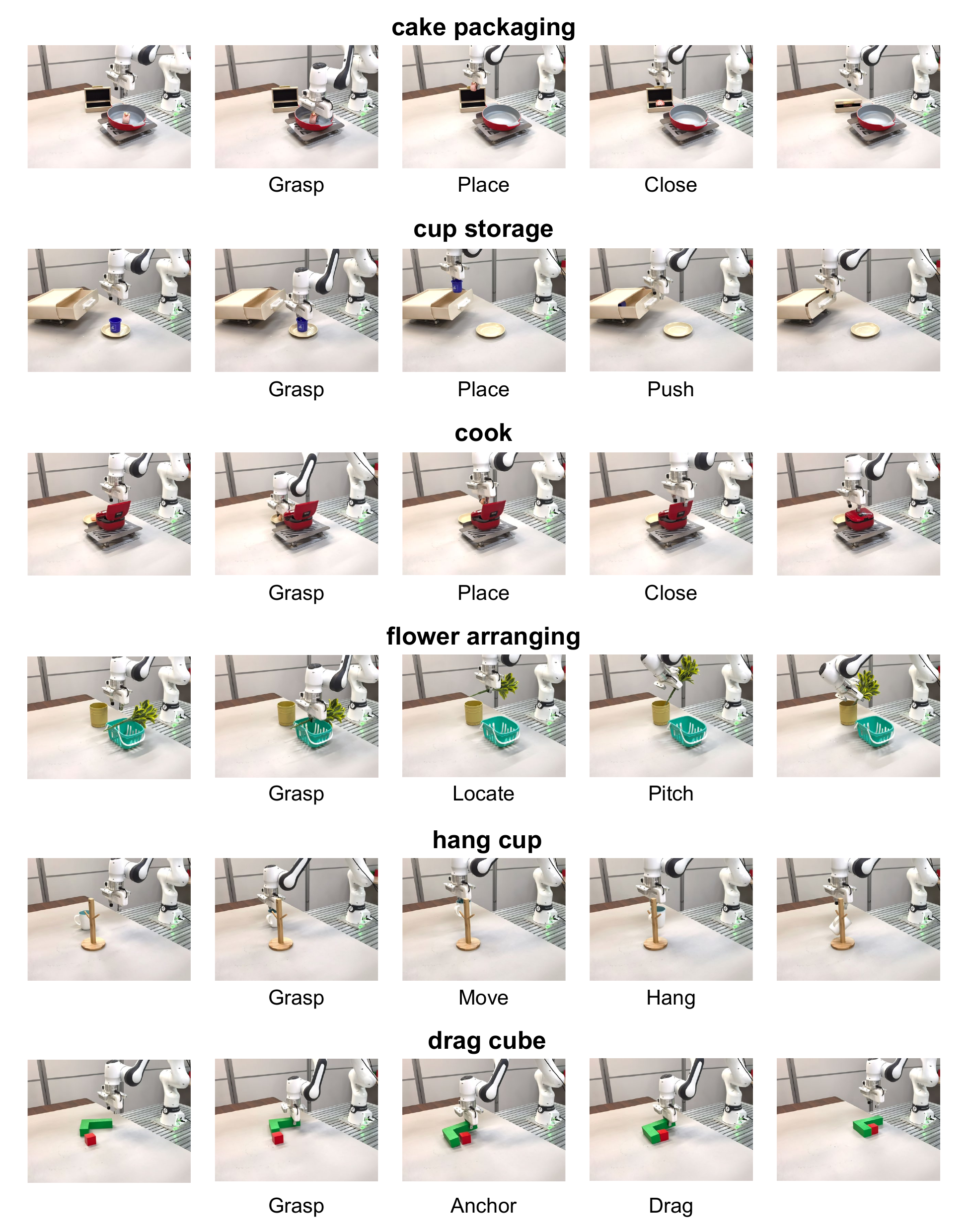}
        \caption{Normal lighting.}
        \label{fig:real_rollouts_normal}
    \end{subfigure}\hfill
    \begin{subfigure}[t]{0.46\linewidth}
        \centering
        \includegraphics[width=\linewidth]{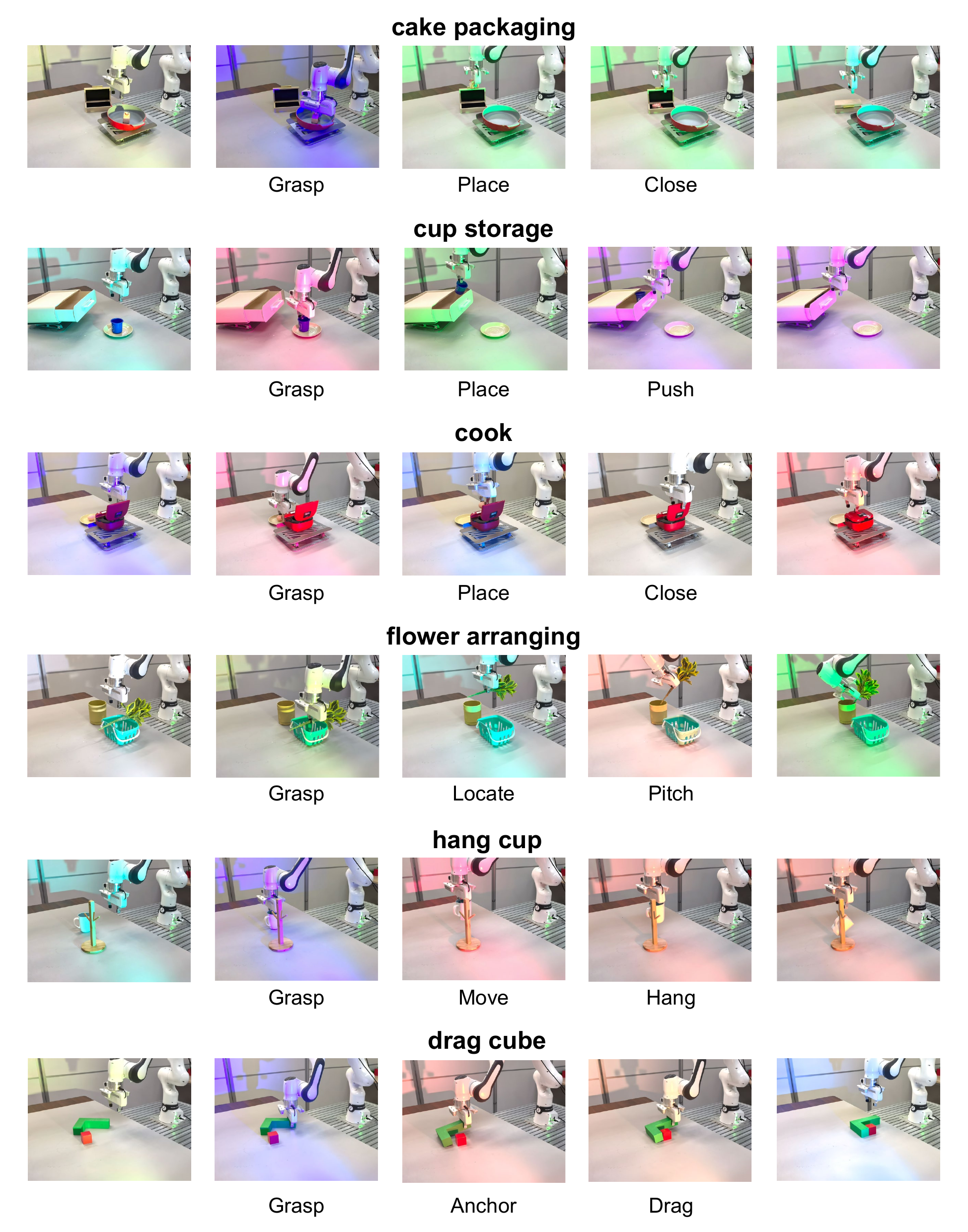}
        \caption{Lighting variation.}
        \label{fig:real_rollouts_light}
    \end{subfigure}
    \captionof{figure}{Real-robot rollouts on the six tasks under normal
    lighting and under varied lighting conditions.
    Each row visualizes the keyframes of one task labeled with the
    active sub-phase.}
    \label{fig:real_rollouts}
\end{figure}

\textbf{Task descriptions.}
We design six long-horizon tasks,
shown in Figure~\ref{fig:task_settings}.
The first four tasks involve object pose changes across episodes,
and the last two have fixed object layouts.
\emph{drag cube} uses a tool to drag a block to a target region.
\emph{hang cup} hangs a cup on a rack.
\emph{flower arranging} inserts a flower into a vase.
\emph{cake packaging} places a cake into a box and closes the
lid.
\emph{cup storage} places a cup into a drawer and closes the
drawer.
\emph{cook} places an egg into a pot and closes the lid.

\textbf{Pose variation.}
Four tasks vary object poses across episodes.
In drag cube the tool is displaced by $\pm5$~cm along $x$ and
rotated by $\pm45$ degrees,
and the cube is displaced by $\pm3$~cm along $y$.
In hang cup either the cup is displaced by about 12~cm along $x$
or the rack is displaced by 5~cm along the diagonal between $-x$
and $-y$.
In flower arranging either the vase is displaced by $\pm8$~cm along
$x$ or the flower and its basket are displaced together by
$\pm5$~cm along $y$.
In cake packaging the box is displaced by 15~cm along $x$ and
rotated by about 90 degrees,
and the plate holding the cake is rotated by about 90 degrees on a
stand inclined by 30 degrees.
Cup storage and cook use fixed layouts,
with the drawer and the cookware placed on stands inclined by 30
degrees,
so three tasks involve non-upright object orientations.

\textbf{Slots and phase names.}
All six tasks contain one grasp-release cycle,
and every demonstration has exactly one open-to-closed gripper
transition,
so the slot count is $K=2$ for every real-robot policy by the rule
$K=2\times\text{max\_cycles}$ of Appendix~\ref{app:slot_warmup}.
The supervision consists only of the four sub-phase classes
$\varphi_t$ derived from the gripper command and the cycle index
$\kappa_t$.
The phase names in Figures~\ref{fig:real_affordance}
and~\ref{fig:real_rollouts} are author annotations that describe the
end-effector state change of the upcoming segment,
and each name maps to the sub-phase whose gripper transition it
describes.
Move before the grasp, push, and close map to $\varphi_t=0$ with the
gripper staying open,
grasp maps to $\varphi_t=1$,
move after the grasp, locate, anchor, and drag map to $\varphi_t=2$
with the gripper staying closed,
and place, hang, and pitch map to $\varphi_t=3$.
The names are therefore neither additional supervision classes nor
fine-grained phase predictions of the model.

\textbf{Rollouts.}
Figure~\ref{fig:real_rollouts} visualizes per-task keyframes for
PAF under normal lighting and under the colored-lighting
variation used to probe robustness.

\subsection{Failure Case Analysis}
\label{app:failure_cases}

\begin{figure}[t]
    \centering
    \includegraphics[width=0.9\linewidth]{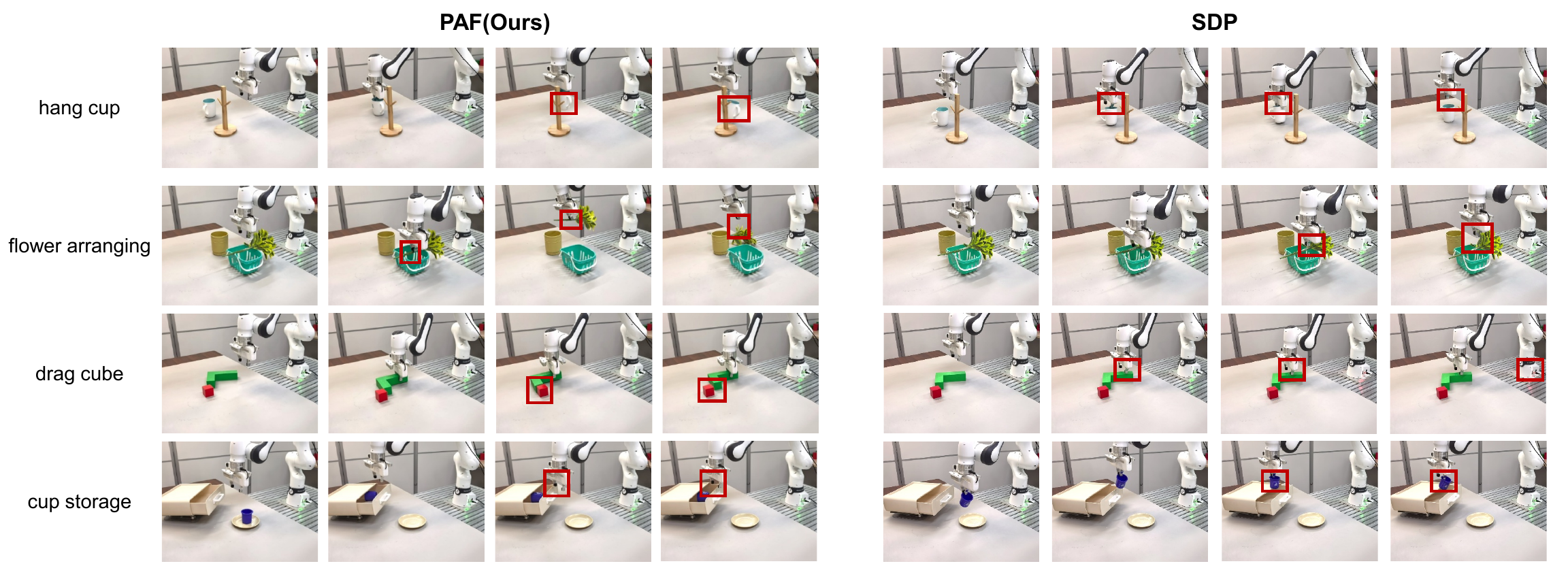}
    \caption{Representative failure cases of PAF and SDP on four
    real-robot tasks.
    Red boxes mark the failure moments.}
    \label{fig:failure}
\end{figure}

Figure~\ref{fig:failure} compares representative failure modes
of PAF and SDP on four tasks.
SDP fails in earlier phases,
missing the cup in \emph{hang cup},
failing to grasp the flower stem in \emph{flower arranging},
colliding with the tool in \emph{drag cube},
and placing the cup inaccurately in \emph{cup storage}.
PAF rarely fails in these earlier phases.
Its failures concentrate in later, contact-rich interactions,
such as the cup loop slipping off the peg in \emph{hang cup},
the stem slipping during rotation in \emph{flower arranging},
the tool failing to hook the block in \emph{drag cube},
and the drawer closing push being misaligned in
\emph{cup storage}.
This indicates that the affordance field provides an effective
geometric inductive bias for spatial localization,
and the remaining headroom lies in fine-grained control during
contact-rich interactions.

\subsection{Inference Overhead}
\label{app:inference_overhead}

Table~\ref{tab:inference_time} reports the per-step inference
time measured on the real-robot deployment configuration.
The affordance field prediction module
adds approximately 100~ms over SDP,
a 13\% overhead relative to the SDP backbone.
The inference bottleneck remains the $SO(3)$-equivariant
diffusion backbone shared by both PAF and SDP,
which accounts for the majority of the 750~ms SDP baseline.
DP3 is faster because it uses a simpler non-equivariant backbone,
at the cost of substantially lower success rates
(Table~\ref{tab:real_robot}).

\begin{center}
\small
\setlength{\tabcolsep}{10pt}
\captionof{table}{Single-step inference time on a single NVIDIA RTX 4090
with DDIM 20-step denoising.}
\label{tab:inference_time}
\begin{tabular}{l c}
\toprule
Method & Inference time (ms) \\
\midrule
DP3        & $\sim$450 \\
SDP        & $\sim$750 \\
PAF (Ours) & $\sim$850 \\
\bottomrule
\end{tabular}
\end{center}

\end{document}